\documentclass{article}
\usepackage{ijcai222}

\usepackage[authoryear,round,longnamesfirst]{natbib}

\usepackage{enumitem}
\usepackage{glossaries}
\usepackage{amsmath} 
\usepackage{amsthm}
\usepackage{multirow}
\usepackage{subcaption}
\usepackage[online]{threeparttable}
\usepackage{xcolor}
\usepackage{tikz}
\usetikzlibrary{positioning}

\usepackage{times}

\usepackage{soul}
\usepackage{url}
\usepackage[hidelinks]{hyperref}
\usepackage[utf8]{inputenc}
\usepackage{caption}
\usepackage{graphicx}
\usepackage{amsmath}
\usepackage{amssymb}
\usepackage{booktabs}
\usepackage{cleveref}
\newtheorem{definition}{Definition}

\newacronym{ad}{AD}{anomaly detection}
\newacronym{occ}{OCC}{one-class classification}
\newacronym{gnn}{GNN}{graph neural network}
\newacronym{gin}{GIN}{graph isomorphism network}
\newacronym{ocgin}{OCGIN}{one-class GIN}
\newacronym{wlk}{WLK}{Weisfeiler-Leman subtree kernel}
\newacronym{pk}{PK}{propagation kernel}
\newacronym{g2v}{G2V}{Graph2Vec}
\newacronym{fgsd}{FGSD}{}
\newacronym{svdd}{deep SVDD}{deep support vector data description}
\newacronym{ocgtl}{OCGTL}{one-class graph transformation learning}
\newacronym{ocgcl}{OCGCL}{one-class graph contrastive learning}
\newacronym{gtl}{GTL}{graph transformation learning}
\newacronym{gtp}{GTP}{graph transformation prediction}
\newacronym{ocpool}{OCPool}{one-class pooling}
\newacronym{ocsvm}{OCSVM}{one-class SVM}
\newacronym{dcl}{DCL}{deterministic contrastive loss}
\newacronym{auc}{AUC}{area under the curve}
\newtheorem{claim}{Proposition}

\Crefname{figure}{Fig.}{Figs.}
\Crefname{table}{Tab.}{Tabs.}
\Crefname{section}{Sec.}{Secs.}
\Crefname{proposition}{Prop.}{Props.}
\Crefname{equation}{Eqn.}{Eqns.}

\title{Raising the Bar in Graph-level Anomaly Detection}

\author{
Chen Qiu$^{1,2}$\and
Marius Kloft$^2$\and
Stephan Mandt$^3$\And
Maja Rudolph$^1$\\
\affiliations
$^1$Bosch Center for Artificial Intelligence\\
$^2$TU Kaiserslautern, Germany\\
$^3$University of California, Irvine, USA\\
\emails
chen.qiu@de.bosch.com,
kloft@cs.uni-kl.de,
mandt@uci.edu,
maja.rudolph@us.bosch.com
}

\begin{document}

\maketitle

\begin{abstract}
Graph-level anomaly detection has become a critical topic in diverse areas, such as financial fraud detection and detecting anomalous activities in social networks. While most research has focused on anomaly detection for visual data such as images, where high detection accuracies have been obtained, existing deep learning approaches for graphs currently show considerably worse performance. This paper raises the bar on graph-level anomaly detection, i.e., the task of detecting abnormal graphs in a set of graphs. 
By drawing on ideas from self-supervised learning and transformation learning, we present a new deep learning approach that significantly improves existing deep one-class approaches by fixing some of their known problems, including hypersphere collapse and performance flip. Experiments on nine real-world data sets involving nine techniques reveal that our method achieves an average performance improvement of $11.8\%$ AUC compared to the best existing approach.
\end{abstract}
\section{Introduction}
\glsresetall
\Gls{ad} is an important tool for scanning systems for unknown threats. 
Many web-based systems are best represented by graphs and there has been work on detecting anomalous nodes and edges within a graph \citep{akoglu2015graph}. 
However, in many applications, it is much more relevant to ask whether an entire graph is abnormal.

For example, in a financial network with nodes representing individuals, businesses, and banks and with edges representing transactions, it might be difficult to detect certain criminal activity by looking at individual nodes and edges 
\citep{jullum2020detecting}. 
Clever criminals can hide their intentions behind innocent-looking transactions. 
However, the entire network associated with a money-laundering scheme is harder to obfuscate and will still exhibit properties of criminal activity. 
By using tools for graph-level \gls{ad}, we might be able to detect an entire criminal network rather than flag individual entities.

Unfortunately, there has been limited success in adapting advances in deep anomaly detection to graph-level \gls{ad} \citep{zhao2021using}. Our work addresses this shortcoming. 
For graph-level \gls{ad}, we assume to have access to a large dataset of typical graphs, such as a dataset of communities in a social network or a dataset of snapshots of a financial network. 
All graphs in the training data are considered ``normal''. The goal is to use the data to learn an anomaly scoring function which can then be used to score how likely it is that a new graph is either normal or abnormal. Importantly, the term graph-level \gls{ad} refers to detecting \emph{entire} abnormal graphs, rather than localizing anomalies within graphs.




Recently there has been a trend of using deep learning in AD on images \citep{golan2018deep} or tabular and sequential data \citep{qiu2021neural}. However, there has been limited research on deep \gls{ad} for graphs. This may seem surprising since it appears straightforward to adopt a deep \gls{ad} method for tabular data into one for graphs by defining an appropriate feature map. Yet, \citet{zhao2021using} found that the resulting methods often perform close to random, and so far, attempts to adopt modern \gls{ad} methods (based on deep learning) to graph-level \gls{ad} have not been successful. 

\begin{figure}[t]
\centering
\includegraphics[width=0.83\linewidth]{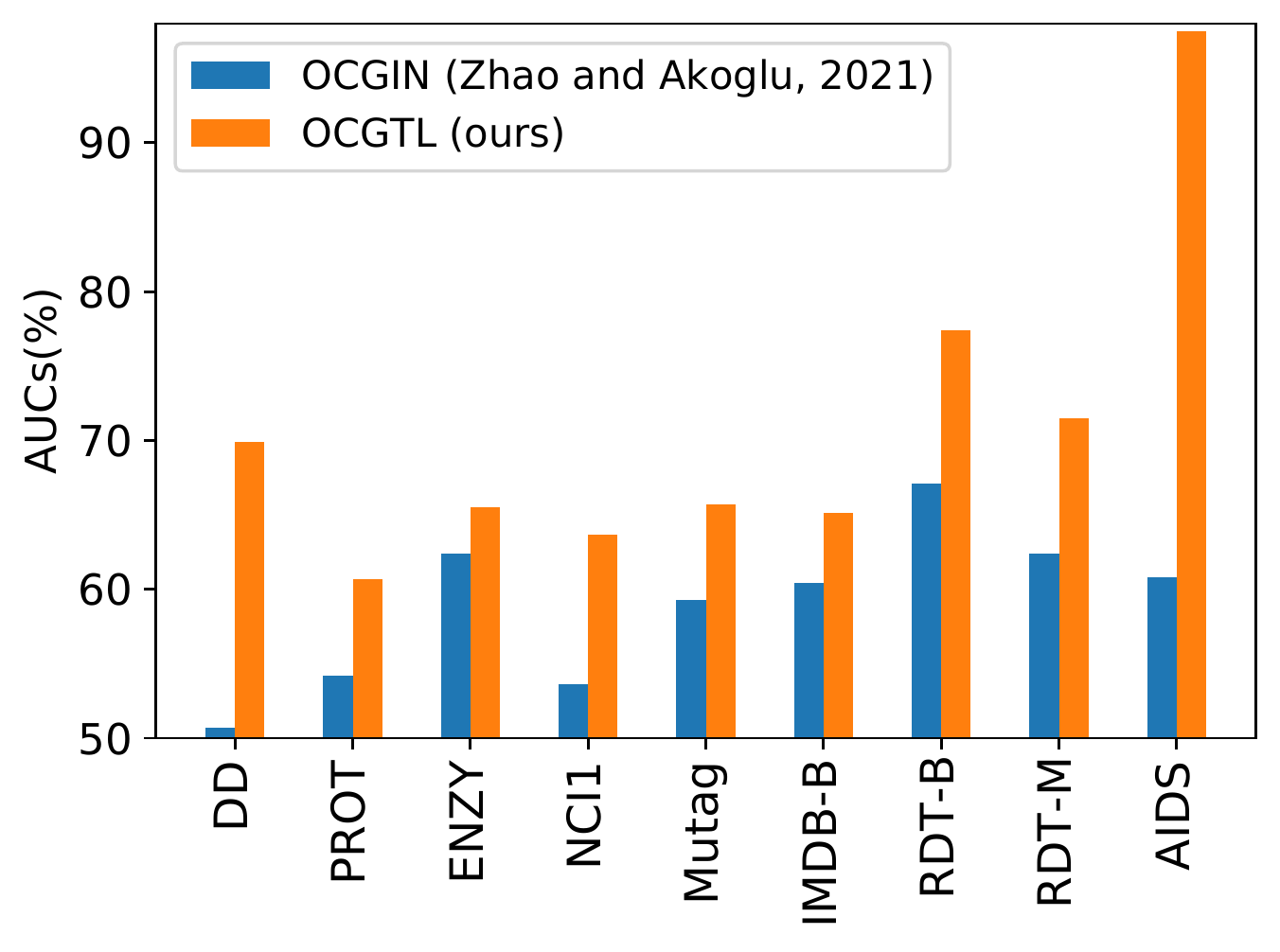}
\caption{AUC comparison between  \gls{ocgtl} (ours) and \Gls{ocgin} \citep{zhao2021using} on several datasets  from \Cref{sec:empirical}. \gls{ocgtl} improves anomaly detection accuracy.}
\label{fig:bar}
\end{figure}
We develop \glsfirst{ocgtl}, a new model for graph level \gls{ad} that combines deep \gls{occ} and self-supervision.  Fig.~\ref{fig:overview} provides a sketch of the approach. The \gls{ocgtl} architecture consists of $K+1$ \glspl{gnn} that are jointly trained on two complementary deep \gls{ad} losses.
In \Cref{sec:theory} we prove that this new combined loss mitigates known issues of previous deep \gls{ad} approaches \citep{ruff2018deep,zhao2021using}. 

\Cref{fig:bar} shows that the approach significantly raises the bar in graph-level AD performance. 
%
%
In an extensive empirical study, we evaluate nine methods on nine real-world datasets. Our work brings deep \gls{ad} on graphs up to speed with other domains, contributes a completely new method (\gls{ocgtl}), and paves the way for future progress.
\begin{figure}[t]
	\centering
	\begin{tabular}{c}
		\includegraphics[width=0.9\linewidth]{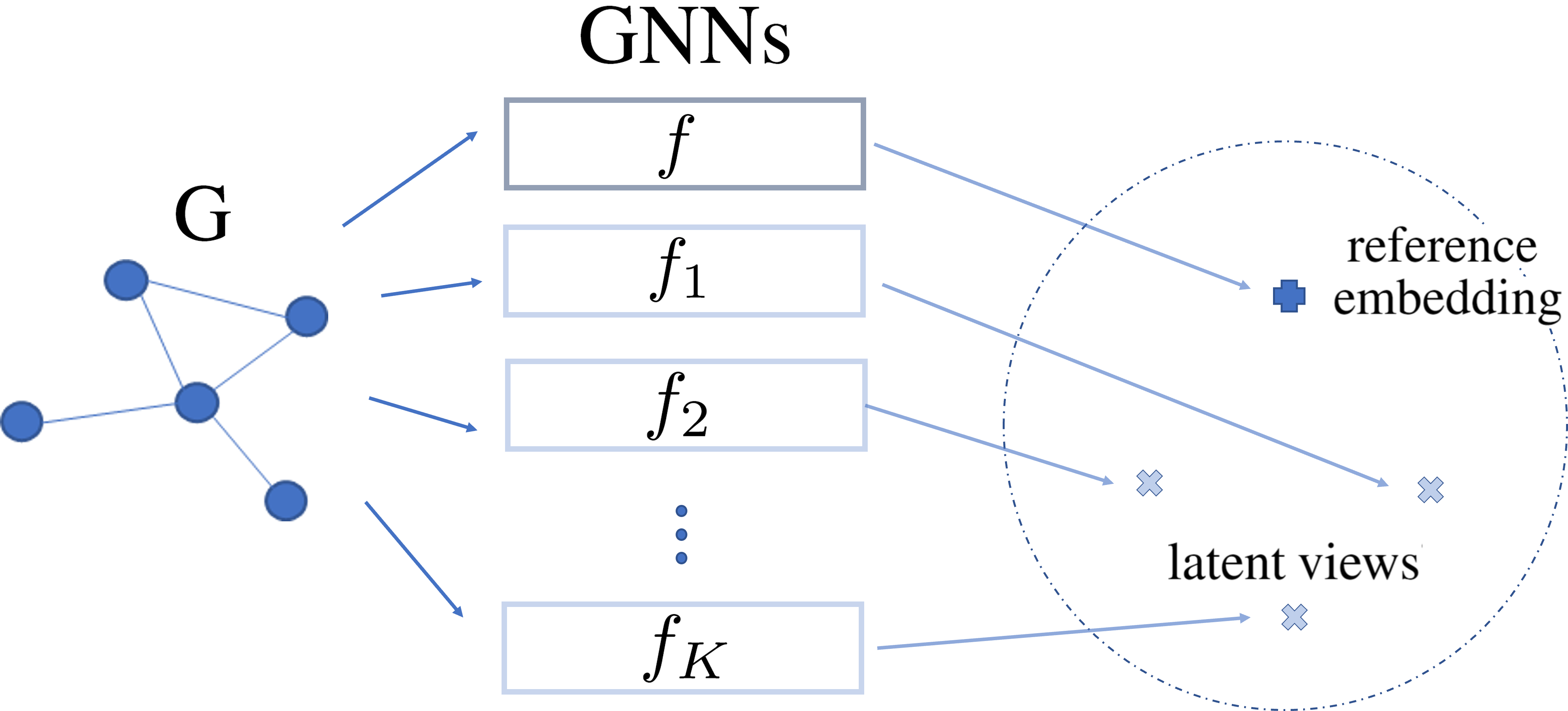}
			\end{tabular}
	\caption{Sketch of the \Gls{ocgtl} procedure. Given a graph (left), we use a set of \glspl{gnn} to embed the latter into a latent space. The different \glspl{gnn} embeddings are trained to be both diverse while also being close to a so-called 'reference embedding'. See \Cref{sec:methodsec} for more details.} 
	\label{fig:overview}
\end{figure}

%


In summary, our main contributions are as follows: 
\begin{itemize}
    \item We develop \gls{ocgtl}, a novel end-to-end method for graph-level \gls{ad}, which combines the advantages of \gls{occ} and neural transformation learning.\footnote{Code is available at \\ \url{https://github.com/boschresearch/GraphLevel-AnomalyDetection}} 
    \item In Sec.~\ref{sec:theory}, we prove theoretically that \gls{ocgtl} is not susceptible to the useless trivial solutions that are optimal under the objective of \citet{zhao2021using}.  
    \item In Sec.~\ref{sec:empirical}, we study nine methods (four newly developed) on nine real-world graph datasets, ranging from social networks to bioinformatics datasets. We improve the architectures of existing deep approaches to graph-level \gls{ad}. Yet, \gls{ocgtl} significantly raises the anomaly detection accuracy over previous work.
\end{itemize}





\section{Related Work}

\paragraph{Deep Anomaly Detection.}
Deep \gls{ad} has received massive attention in various domains \citep{ruff2021unifying}.
Related work on deep \gls{ad} can be summarized in the following classes. 
Deep autoencoder variants \citep{zong2018deep} detect anomalies based on the reconstruction error.
Deep one-class networks \citep{ruff2018deep} are trained on an \gls{occ} objective to map normal data close to a center in an embedding space. They score anomalies based on their distance to the center. 
Deep generative models detect anomalies based on density estimation \citep{zhang2021understanding} or using the discriminator of generative adversarial networks \citep{deecke2018image}.
Self-supervised \gls{ad} methods have achieved great success on images \citep{ruff2021unifying}. They rely on a self-supervision task for training and anomaly scoring. The self-supervision tasks often require data augmentation, e.g., for transformation prediction \citep{golan2018deep}. Another self-supervised paradigm using data augmentations is contrastive learning \citep{chen2020simple}. In contrastive learning for \gls{ad}, transformed (e.g. rotated) images are treated as negative samples \citep{sohn2020learning}. For data types beyond images, handcrafting transformations is challenging. Data-driven neural transformations \citep{qiu2021neural} have shown success in \gls{ad}. 

\paragraph{Graph Anomaly Detection.}
Finding abnormal nodes or edges within a large graph is widely studied~\citep{akoglu2015graph}. Deep learning based methods for node- and edge-level \gls{ad}
have been successfully applied to both static
\citep{ding2019deep,ding2020inductive} and dynamic graphs \citep{yoon2019fast,zheng2019addgraph}. In contrast, deep learning for graph-level \gls{ad}, the problem we study in this paper, has received less attention. \citet{zhao2021using} first explore how to extend deep \gls{occ} for graph-level \gls{ad} and develop \glsfirst{ocgin}. They also introduce two-stage graph-level \gls{ad} frameworks with either graph embedding models or graph kernels. However, all these attempts have not yet produced a solid baseline for graph-level \gls{ad}. \citet{zhao2021using} report that these methods suffer from ``performance flip'', which is an issue where the model performs worse than random on at least one experimental variant. We study how to overcome it.

\paragraph{Deep Learning for Graphs.}
\Glspl{gnn} automate feature extraction from graphs and play an important role in graph classification and representation learning. 
Graph Convolution Networks \citep{kipf2016semi} learn a node representation by aggregating the representations of the node's neighbors. 
\citet{errica2020fair} provide a fair comparison of various \glspl{gnn} for graph classification.
An outstanding example is \gls{gin} \citep{xu2018powerful}, which is proven to be as powerful as the Weisfeiler-Lehman graph isomorphism test and is the architecture we use for all \gls{gnn}-based \gls{ad} methods we study in this paper. 

\section{One-Class Graph Transformation Learning for Anomaly Detection}
\label{sec:methodsec}
We propose a new method for graph-level \gls{ad}. \Glsfirst{ocgtl} combines the complementary strengths of deep \gls{occ} \citep{ruff2018deep,zhao2021using} and self-supervised \gls{ad} with learnable transformations \citep{qiu2021neural}.

Our new self-supervised approach is designed to overcome known issues of deep \gls{occ}.
The deep \gls{occ} objective is prone to a trivial solution called {\em hypersphere collapse}.
The deep one-class objective encourages all the graph embeddings of the graphs in the training data to concentrate within a hypersphere. This task can be solved perfectly when the feature extractor learns to map all inputs to the center of the hypersphere.
Our model provably overcomes hypersphere collapse by regularizing the one-class terms in the objective with a transformation learning term. The resulting model is more flexible (for example, the hypersphere centers can be treated as trainable parameters) and training is more robust, despite the added flexibility.

\citet{zhao2021using} have developed the first deep one-class approach to graph-level \gls{ad}. In their paper, they report an additional, practical difficulty in graph-level \gls{ad} which they call the {\em performance flip issue}. In many of their experiments, their trained model (\gls{ocgin}) systematically confuses anomalies with normal samples.
The goal of this work is to overcome both hypersphere collapse and performance flip. 

Our model consists of an ensemble of \glspl{gnn}. One of them -- the reference feature extractor -- produces a reference embedding of its input graph. The other \gls{gnn} feature extractors produce alternative ``latent views'' of the graph.
The objective of our approach has a one-class term and a transformation learning term. The one-class term aims at concentrating all the latent views within a hyper-sphere in the embedding space. Transformation learning has the competing objective to make each view predictive of the reference embedding. It encourages the latent views to be diverse yet semantically meaningful. 
By counteracting the one-class term in this manner, hypersphere collapse can be provably avoided.

The tension that arises from satisfying both aspects of the objective has further advantages.
In particular, it leads to a  harder self-supervision task, which in turn leads to better anomaly detection performance. When the training objective is difficult to satisfy, the trained model has to be more sensitive to typical salient features of normal data. New graphs which do not exhibit these features incur a higher loss and are then more easily detected as anomalies.  Also, the two loss contributions focus on different notions of distance between the graph embeddings. The one-class term is based on Euclidean distances, while the transformation learning loss is based on angles between embeddings. With the combined loss as the anomaly score, our method is sensitive to abnormal embedding configurations both in terms of angles between the latent views and in terms of Euclidean distances.

In this section, we first introduce \gls{ocgtl} and then detail its main ingredients, including self-supervised \gls{ad} with learnable transformations, deep \gls{occ}, and feature extraction with \glspl{gnn}. We then present the theory behind \gls{ocgtl}.
\subsection{Proposed Method - OCGTL}
\label{sec:OCGTL}
\gls{ocgtl} combines the best of \gls{occ} and neural transformation learning. 
The \gls{ocgtl} architecture consists of a reference feature extractor $f$ and $K$ additional feature extractors $f_k$ ($k=1,\cdots, K$), which are trained jointly as illustrated in \Cref{fig:overview}. Each of the feature extractors is a parameterized function (e.g. \gls{gnn}) which takes as input an attributed graph  $\textrm{G} = \{\mathcal{V},\mathcal{E},\mathcal{X}\}$, with vertex set $\mathcal{V}$, edges $\mathcal{E}$, and node features (attributes) $ \mathcal{X} = \{ x_v | v \in \mathcal{V} \}$ and maps it into an embedding space $\mathcal{Z}$. These $K+1$ feature extractors are trained jointly on the \gls{ocgtl} loss, $\mathcal{L}_{\textrm{OCGTL}}= \mathbb{E}_\textrm{G}\left[\mathcal{L}_{\textrm{OCGTL}}(\textrm{G}) \right]$. Each graph in the training data contributes two terms to the loss,  
\begin{align}
\label{eqn:ocgtl}
    \mathcal{L}_{\textrm{OCGTL}}(\textrm{G}) =  \mathcal{L}_{\textrm{OCC}}(\textrm{G})+ \mathcal{L}_{\textrm{GTL}}(\textrm{G}).
\end{align}
The first term, $\mathcal{L}_{\textrm{OCC}}(\textrm{G})$, is a one-class term; it encourages all the embeddings to be as close as possible to the same point $\theta \in \mathcal{Z}$. The second term, $\mathcal{L}_{\textrm{GTL}}$, enforces each \gls{gnn}'s embeddings to be diverse and semantically meaningful representations of the input graph $\textrm{G}$. 

The two terms are presented in detail below.

\subsubsection{The Graph Transformation Learning Term}
\label{sec:AUG}
Neural transformation learning \citep{qiu2021neural} is a self-supervised training objective for deep \gls{ad} which has seen success on time series and tabular data. Here we generalize the training objective of \citet{qiu2021neural} (by dropping their parameter sharing constraint) and adapt it to graphs.

For a graph $\textrm{G}$, the loss of graph transformation learning encourages the embeddings of each \gls{gnn}, $f_k(\textrm{G})$, to be similar to the embedding of the reference \gls{gnn}, $f(\textrm{G})$, while being dissimilar from each other. Consequently, each \gls{gnn} $f_k$ is able to extract graph-level features to produce a different view of $\textrm{G}$. The contribution of each graph to the objective is
\begin{align}
\label{eqn:GTL_loss}
 & \mathcal{L}_{\textrm{GTL}}(\textrm{G})  = -\sum_{k=1}^K \log \frac{c_k}{C_k}\\\nonumber
 \text{with} \qquad c_k & =  \exp\left(\frac{1}{\tau} \mathrm{sim}(f_k(\textrm{G}), f(\textrm{G}))\right), \\\nonumber  C_k & = c_k +\sum_{l\neq k}^K \exp\left(\frac{1}{\tau}\mathrm{sim}(f_k(\textrm{G}), f_l(\textrm{G}))\right),
\end{align}
where $\tau$ denotes a temperature parameter. The similarity here is defined as the cosine similarity $\mathrm{sim}(z,z') := z^T z'/ \|z\| \|z'\|$. 
Note that the above loss is more general than the one proposed in \citet{qiu2021neural} as it omits a parameter sharing constraint between transformations.
This choice is inspired by the observation in \citet{you2020graph} that different graph categories prefer different types of transformations. 

\subsubsection{The One-Class Term}
\label{sec:OCC}
\Glsfirst{occ} is a popular paradigm for \gls{ad} \citep{noumir2012simple}. The idea is to map data into a minimal hypersphere encompassing all normal training data. Data points outside the boundary are considered anomalous. 
The contribution of each graph $\textrm{G}$ to our \gls{occ} objective is
\begin{align}
\label{eqn:oc_loss}
 \mathcal{L}_{\textrm{OCC}}(\textrm{G}) =
\sum_{k=1}^K \|(f_k(\textrm{G})-\theta)\|_2
\end{align}
The loss function penalizes the distance of the graph $\textrm{G}$ to the center $\theta$ which we treat as a trainable parameter. In previous deep \gls{occ} approaches, the center $\theta$ has to be a fixed hyperparameter to avoid trivial solutions to \Cref{eqn:oc_loss}.

\subsubsection{Feature Extraction with GNNs}
\label{sec:GNN}
For graph data, parametrizing the feature extractors $f$ and $f_1, \cdots, f_K$ by \glspl{gnn} is advantageous. At each layer $l$, a \gls{gnn} maintains node representation vectors $h_v^{(l)}$ for each node $v$.
The representation is computed based on the previous layer's representations of $v$ and its neighbors $\mathcal{N}(v)$,
\begin{align}
\label{eqn:gnn_layer}
    h_v^{(l)} = \textrm{GNN}^{(l)}\left(h_v^{(l-1)},  h_u^{(l-1)} \mid u\in \mathcal{N}(v)\right)\,.
\end{align}
Each layer's node representations are then combined into layer-specific graph representations,
\begin{align}
\label{eqn:readout}
     h_{\textrm{G}}^{(l)} =\textrm{\small READOUT}^{(l)}\left(h_v^{(l)} \mid v\in G \right),
\end{align}
which are concatenated into graph-level representations,
\begin{align}
\label{eqn:concat}
    h_{\textrm{G}}  =  \textrm{\small CONCAT}\left(h_{\textrm{G}}^{(l)} \mid l = 1,...,L\right).
\end{align}
This concatenation introduces information from various hierarchical levels \citep{xu2018representation} into the graph representation.
Our empirical study in \Cref{sec:empirical} shows that the choice of the readout function (which determines how the node representations are aggregated into graph representations) is particularly important to detect anomalies reliably.

\subsubsection{Anomaly Scoring with OCGTL}
\Gls{ocgtl} is an end-to-end methods for graph-level \gls{ad}. During training the \glspl{gnn} are trained on \Cref{eqn:ocgtl}. 
During test, $\mathcal{L}_{\textrm{OCGTL}}$ (\Cref{eqn:ocgtl}) is used directly as the score function for detecting anomalous graphs. A low loss on a test sample means that the graph is likely normal, whereas a high loss is indicative of an anomaly. One advantage of \gls{ocgtl} is that its loss makes it more sensitive to different types of anomalies by considering both angles between embeddings and Euclidean distances. In contrast,  \gls{occ}-based methods typically rely on the Euclidean distance only.

Another advantage of \gls{ocgtl} over \gls{occ}-based approaches is that its training is more robust and the \gls{ad} model can be more flexible. We prove this next.

\subsection{A Theory of OCGTL}
\label{sec:theory}
A known difficulty for training \gls{occ}-based deep anomaly detectors (such as deep SVDD and \gls{ocgin})
is \emph{hypersphere collapse} \citep{ruff2018deep}. Hypersphere collapse is a trivial optimum of the training objective \begin{align}
\label{eqn:svdd_loss}
    \mathcal{L}_{\textrm{\citep{ruff2018deep}}}(\textrm{G}) =  ||f(\textrm{G})-\theta||_2^2 \,,
\end{align}
which occurs when the feature extractor $f$ maps all inputs exactly into the center $\theta$. The hypersphere then has a radius of zero, and \gls{ad} becomes impossible. 
\citet{ruff2018deep} recommend fixing $\theta$ and avoiding bias terms for $f$ and show good results in practice.
However, there is no guarantee that a trivial solution can be avoided under any architecture for $f$. Here we prove that \gls{ocgtl} overcomes this.

We first show that our one-class term (\Cref{eqn:oc_loss}) is also prone to hypersphere collapse when all the feature extractors are constant. However, we then show that this trivial solution for minimizing \Cref{eqn:oc_loss}  is not optimal under the \gls{ocgtl} loss. Our method provably avoids hypersphere collapse even when the center $\theta$ is a trainable parameter. This result makes \gls{ocgtl} the first deep one-class approach where the center can be trained.

\begin{claim}
\label{c1}
 The constant feature extractors, $f_k(\textrm{G}) = \theta$ for all $k$ and all inputs $\textrm{G}$, minimize $\mathcal{L}_{\textrm{OCC}}$ (\Cref{eqn:oc_loss}).
\end{claim}
\begin{proof}
$0 \leq \mathcal{L}_{\textrm{OCC}}$ is the squared $\ell_2$ norm of the distance between the embedding of $\textrm{G}$ and the center $\theta$. Plugging in $f_k(\textrm{G}) = \theta$ attains the minimum $0$.
\end{proof}
In contrast, regularization with transformation learning can avoid hypersphere collapse. Under the constant encoder, all the latent views are the same and hence at least as close to each other as to the reference embedding, leading to $\mathcal{L}_{\textrm{GTL}} \geq K \log K$. However, the transformation learning objective aims at making the views predictive of the reference embeddings, in which case $\mathcal{L}_{\textrm{GTL}} < K \log K$. The following proposition shows that if there is a parameter setting which achieves this, the constant feature extractors do not minimize the \gls{ocgtl} loss which proves that hypersphere collapse can be avoided.

\begin{claim}
\label{c2}
If there exists a parameter setting such that $\mathcal{L}_{\textrm{GTL}} < K \log K$ on the training data, then the constant feature extractors $f_k(\textrm{G}) = \theta$ do not minimize the combined loss $\mathcal{L}_{\textrm{OCGTL}}$ (\Cref{eqn:ocgtl}).
\end{claim}
\begin{proof}
For constant feature extractors $f_k(\textrm{G}) = \theta$, $\mathcal{L}_{\textrm{OCGTL}} = \mathcal{L}_{\textrm{GTL}} \geq K\log K$, where $K$ is the number of transformations and $K\log K$ is the negative entropy of randomly guessing the reference embedding.
Assume there is a constellation of the model parameters s.t. $\mathcal{L}_{\textrm{GTL}} < K\log K$. Since $\theta$ is trainable, we can set it to be the origin. The loss of the optimal solution is at least as good as the loss with $\theta=0$. 
Set $\epsilon =  K\log K - \mathcal{L}_{\textrm{GTL}} $. The encoders can be manipulated such that their outputs are rescaled and as a result all the embeddings have norm $||f_k(\textrm{G})||_2 < \epsilon / K$.
As the norm of the embeddings changes, $\mathcal{L}_{\textrm{GTL}}$ remains unchanged since the cosine similarity is not sensitive to the norm of the embeddings. By plugging this into \Cref{eqn:ocgtl} we get 
$\mathcal{L}_{\textrm{OCGTL}} = \sum_{k=1}^K ||f_k(\textrm{G})||_2 + \mathcal{L}_{\textrm{GTL}} < K\log K$,
which is better than the performance of the best set of constant encoders.
\end{proof}

Props. 1 and 2 demonstrate that our method is the first deep one-class method not prone to hypersphere collapse. The assumption of Prop. 2, that $\mathcal{L}_{\textrm{GTL}}< K \log K$ can be tested in practice by training \gls{gtl} and evaluating the predictive entropy on the training data. In all scenarios we worked with $\mathcal{L}_{\textrm{GTL}} << K \log K$ after training.

\subsection{Newly Developed Baselines}
\label{sec:methods_summary}
The main contribution of our work is \gls{ocgtl}.
To study the effectiveness of \gls{ocgtl} we implement the following graph-level \gls{ad} methods as ablations. 
These methods have not been studied on graphs before, so their implementation is also one of our contributions that paves the way for future progress.

\paragraph{\Gls{ocpool}.} As a shallow method, \gls{ocpool} uses pooling to construct a graph representation: 
\begin{align}
    h_{\textrm{G}} =\textrm{POOLING}\left(x_v \mid v\in \textrm{G}\right).
\end{align}
This feature extractor does not have parameters and hence requires no training.
Anomalies can be detected by training an \gls{ocsvm} \citep{manevitz2001one} on these features.
This novel approach for graph-level \gls{ad} is a simple baseline and achieves solid results in our empirical study (even though it does not use the edge sets $\mathcal{E}$ of the graphs).
Another reason for studying \gls{ocpool}, is that it helps us understand which pooling function might work best as readout function (\Cref{eqn:readout}) for \gls{gnn}-based \gls{ad} methods.

\paragraph{\Gls{gtp}.} \Gls{gtp}
is an end-to-end self-supervised detection method based on transformation prediction. 
It trains a classifier $f$ to predict which transformation has been applied to a samples and uses the cross-entropy loss to score anomalies.
We implement \gls{gtp} with six graph transformations (node dropping, edge adding, edge dropping, attribute masking, subgraph, and identity transformation) originally designed in \citet{you2020graph}.

\paragraph{\Glsfirst{gtl}.} \Gls{gtl}
is an end-to-end self-supervised detection method using neural transformations \citep{qiu2021neural}. $K$ \glspl{gnn}, $f_k$  for $k=1,\cdots, K$ in addition to the reference feature extractor $f$ are trained on $\mathcal{L}_{\textrm{GTL}}$ (\Cref{eqn:GTL_loss}). The loss is used directly to score anomalies. While this method works well in practice, it is not sensitive to the norm of the graph embeddings in \Cref{eqn:GTL_loss}. The normalization step in computing the cosine similarity makes mean and add pooling equivalent when aggregating the graph representations.
This may put \gls{gtl} at a disadvantage compared to the other methods, which profit from add pooling.

\section{Experiments}
\label{sec:empirical}
This section details our empirical study. We benchmark nine algorithms on nine real-world graph classification datasets from different domains using various evaluation measures. First, we describe the datasets and how the \gls{ad} benchmark is set up. Second, we present all methods we compare, including baselines and their implementation details. Third, the evaluation results are presented and analyzed. In summary, \gls{ocgtl} achieves the best performance on real-world datasets from various domains and raises the anomaly detection accuracy significantly ($+11.8\%$ in terms of AUC on average compared to \gls{ocgin} of \citet{zhao2021using}). 
Finally, we present our findings about preferable design choices that are also beneficial for other deep methods in graph-level \gls{ad}. 

\begin{table*}[ht!]\small
	\centering
	\begin{tabular}{cc|ccc|ccc|c}
 & &DD &  PROT & ENZY & NCI1 & AIDS &  Mutag & Rank\\

		\hline
\multirow{5}{*}{\rotatebox{90}{Baselines} \rotatebox{90}{(prev. work)}}&\acrshort{wlk}& 50.2$\pm$0.3$^*$	& 49.7$\pm$0.5$^*$	& 52.1$\pm$2.0$^*$	& 49.6$\pm$0.4$^*$	& 51.3$\pm$0.8$^*$	& 52.3$\pm$0.6 &8.1$\pm$1.7	\\
&\acrshort{pk}& 51.2$\pm$2.3$^*$	& 50.8$\pm$1.5$^*$	& 51.3$\pm$1.2$^*$	& 51.4$\pm$1.7$^*$	& 59.5$\pm$2.3$^*$	& 52.5$\pm$1.6	& 8.4$\pm$1.1\\
&\acrshort{g2v}& 49.5$\pm$2.2$^*$	& 53.2$\pm$3.0$^*$	& 52.0$\pm$3.9$^*$	& 50.5$\pm$0.7$^*$	& 48.4$\pm$0.8$^*$	& 50.2$\pm$0.7$^*$	&8.9$\pm$1.0\\	
&\acrshort{fgsd}& 66.0$\pm$2.3	& 58.5$\pm$1.8$^*$	& 52.7$\pm$3.4$^*$	& 55.4$\pm$0.7	& 91.6$\pm$3.7	& 51.3$\pm$0.8$^*$	& 5.3$\pm$2.9\\
&\acrshort{ocgin} & 50.7$\pm$1.2$^*$	& 54.2$\pm$1.2$^*$	& 62.4$\pm$2.7	& 53.6$\pm$1.3$^*$	& 60.8$\pm$2.2$^*$	& 59.3$\pm$1.4 & 5.4$\pm$1.7\\
\hline
\multirow{4}{*}{\rotatebox{90}{Ablations } \rotatebox{90}{ (ours)}}&OCGIN$^\dagger$ & 61.4$\pm$1.6	& 57.2$\pm$2.3	& 63.5$\pm$3.9	& 62.2$\pm$1.3	& \textbf{97.5$\pm$2.0}	& 61.5$\pm$1.8$^*$	& 2.7$\pm$0.9\\
&OCPool& 61.1$\pm$3.3$^*$	& \textbf{61.9$\pm$2.1}	& 53.1$\pm$2.9$^*$	& 57.0$\pm$1.3	& \textbf{97.6$\pm$1.4}	& 53.2$\pm$0.5$^*$	& 4.3$\pm$2.2\\
&GTP& 54.2$\pm$1.9	& \textbf{61.9$\pm$2.9}	& 55.0$\pm$2.0$^*$	& 55.3$\pm$1.2$^*$	& 77.2$\pm$2.9	& 54.7$\pm$1.8$^*$ &4.8$\pm$1.5\\
&GTL& 51.7$\pm$0.9$^*$	& 56.2$\pm$2.5$^*$	& 60.4$\pm$1.6	& 59.8$\pm$1.0$^*$	& 67.8$\pm$3.3$^*$	& 61.8$\pm$1.0	& 3.8$\pm$1.7\\
\hline
ours&OCGTL& \textbf{69.9$\pm$2.6}	& 60.7$\pm$2.4	& \textbf{65.5$\pm$3.8}	& \textbf{63.7$\pm$1.2}	& \textbf{97.5$\pm$2.0}	& \textbf{65.7$\pm$2.1}	&1.4$\pm$0.7\\
        \hline
	\end{tabular}
		      \begin{tablenotes}
      \small
      \item  $^\dagger$ OCGIN is the original implementation from \citet{zhao2021using}, while OCGIN$^\dagger$ denotes our improved implementation with the same \gls{gin} architecture choices (add pooling etc.) as \gls{ocgtl}, \acrshort{gtp}, and \acrshort{gtl}.
    \end{tablenotes}
		\caption{Average AUCs ($\%$) with standard deviations of 9 methods on 6 of 9 datasets. (For the remaining 3 datasets, see \Cref{tab:social_network_results}.) The performance rank averaged on all nine datasets is provided in the last column. Results marked $^*$ perform worse than random on at least one experimental variant (performance flip). \Gls{ocgtl} outperforms the other methods and has no performance flip.}
	\label{tab:glad_results}
\end{table*}
\begin{table*}[ht!]
    \small
	\centering
	\begin{tabular}{c|cc|cc|cc|cc}
 & \multicolumn{2}{c|}{DD} &  \multicolumn{2}{c|}{PROT} & \multicolumn{2}{c|}{NCI1} & \multicolumn{2}{c}{AIDS}\\
Outlier class& 0 & 1 &0 & 1 & 0 & 1 & 0 & 1\\
		\hline
\gls{ocgin} & \textcolor{red}{26.3$\pm$2.7} & 75.2$\pm$3.4 &\textcolor{red}{42.5$\pm$4.4} & 65.9$\pm$4.5 & 64.4$\pm$2.5 & \textcolor{red}{42.9$\pm$3.0} & \textcolor{red}{26.0$\pm$3.9} & 95.5$\pm$2.2\\
\gls{ocgtl} (ours)& 66.8$\pm$4.6 & 73.0$\pm$2.9 & 63.2$\pm$5.4 & 58.1$\pm$6.1 & 71.2$\pm$3.0 & 56.2$\pm$2.5 & 99.3$\pm$0.9 & 95.7$\pm$3.6\\ 
        \hline
	\end{tabular}
	\caption{Average AUCs ($\%$) with standard deviations of \gls{ocgin} \citep{zhao2021using} and \gls{ocgtl} (ours) on both experimental variants of four datasets, where the performance flip is observed. The results that are worse than random are marked in red. \gls{ocgin} suffers from performance flip, while \gls{ocgtl} not.}
	\label{tab:detailed_results}
\end{table*}

\begin{table}[ht!]
\small
	\centering
	\begin{tabular}{cc|ccc}
 & & IMDB-B  & RDT-B & RDT-M \\

		\hline
\multirow{5}{*}{\rotatebox{90}{Baselines} \rotatebox{90}{(prev. work)}}&\acrshort{wlk}& 62.0$\pm$2.0	& 50.2$\pm$0.2$^*$	& 50.3$\pm$0.1$^*$ \\
&\acrshort{pk}& 53.5$\pm$2.0	&50.0	&50.0	\\
&\acrshort{g2v}& 54.3$\pm$1.6	& 51.5$\pm$0.6$^*$	& 50.0$\pm$0.2$^*$ \\	
&\acrshort{fgsd}& 57.1$\pm$1.8	&-	&-	\\
&\acrshort{ocgin} &  60.4$\pm$2.8	& 67.1$\pm$3.5	& 62.4$\pm$1.3 \\
\hline
\multirow{4}{*}{\rotatebox{90}{Ablations } \rotatebox{90}{ (ours)}}&OCGIN$^\dagger$ & 63.7$\pm$2.4	& 74.5$\pm$3.4	& 70.4$\pm$1.5	\\
&OCPool& 56.5$\pm$0.8	& 65.3$\pm$2.2	& 62.4$\pm$0.9	\\
&GTP& 57.6$\pm$1.1	& 64.4$\pm$2.2	& 62.4$\pm$1.3\\
&GTL&\textbf{65.2$\pm$1.9}	& 71.6$\pm$2.3	& 67.8$\pm$0.9\\
\hline
ours&OCGTL&  \textbf{65.1$\pm$1.8}	& \textbf{77.4$\pm$1.9}	& \textbf{71.5$\pm$1.1}\\
        \hline
	\end{tabular}
	\begin{tablenotes}
      \small
      \item  $^\dagger$ OCGIN is the original implementation from \citet{zhao2021using}, while OCGIN$^\dagger$ denotes our improved version.
    \end{tablenotes}
    	\caption{Average AUCs ($\%$) with standard deviations of nine methods on three datasets to complement \Cref{tab:glad_results}.}
	\label{tab:social_network_results}
	\vspace{-8pt}
\end{table}
\subsection{Datasets and Experimental Setting}
We benchmark nine methods on nine graph classification datasets that are representative of three domains. In addition to financial and social networks security, health organizations need an effective graph-level \gls{ad} method to examine proteins (represented as graphs) to monitor the spread and evolution of diseases. 
Targeting these application domains, we study three bioinformatics datasets: DD, PROTEINS, and ENZYMES, three molecular datasets:  NCI1, AIDS, and Mutagenicity, and three datasets of social networks: IMDB-BINARY, REDDIT-BINARY, and REDDIT-MULTI-5K. The datasets are made available by \citet{Morris+2020}, and the statistics of the datasets are given in Appendix A.

We follow the standard setting of previous work to construct an \gls{ad} task from a classification dataset \citep{ruff2018deep,golan2018deep,zhao2021using}. A classification dataset with $N$ classes produces $N$ experimental variants. In each experimental variant, one of the classes is treated as ``normal''; the other classes are considered as anomalies.
The training set and validation set only contain normal samples, while the test set contains a mix of normal samples and anomalies that have to be detected during test time. For each experimental variant, $10\%$ of the normal class is set aside for the test set, and $10\%$ of each of the other classes is added to the test set as anomalies. (The resulting fraction of anomalies in the test set is proportional to the class balance in the original dataset. The remaining $90\%$ of the normal class is used for training and validation.
We use 10-fold cross-validation to estimate the model performance. In each fold, $10\%$ of the training set is held out for validation. 
We train each model three times separately and average the test results of three runs to get the final test results in each fold. Training multiple times ensures a fair comparison as it favors methods that are robust to the random initialization.

\paragraph{Evaluation.}
Results will be reported in terms of the area under the ROC curve (AUC) ($\%$), averaged over 10 folds with standard deviation. We also report the results in terms of F1-score in Appendix C. 
In addition, all methods will be evaluated in terms of their susceptibility to performance flip.

\citet{zhao2021using} coined the term ``performance flip'' for \gls{ad} benchmarks derived from binary classification datasets. 
We generalize their definition to multiple classes:
\begin{definition}
(Performance flip.) A model suffers from performance flip on an anomaly detection benchmark derived from a classification dataset if it performs worse than random on at least one experimental variant. 
\end{definition} 

\subsection{Baselines and Implementation Details}
\label{sec:baselines}
Many deep \gls{ad} approaches that have achieved impressive results in other domains have not yet been adapted to graph-level \gls{ad}. There has been no comprehensive study of various \gls{gnn}-based graph-level \gls{ad} approaches. An additional contribution of our work is that we adapt recent advances in deep \gls{ad} to graphs.
In our empirical study, we compare \gls{ocgtl} both to \gls{gnn}-based methods and to non-\gls{gnn}-based methods, which we outline below.

\paragraph{GNN-based Baselines.}
Our study includes \gls{ocgtl}, \gls{ocgin} \citep{zhao2021using} and the self-supervised approaches \glsfirst{gtp} and \glsfirst{gtl} described in \Cref{sec:methods_summary}.
We use \gls{gin} as the feature extractor for all \gls{gnn}-based baselines to compare with \gls{ocgin} fairly. 
In particular, we use 4 \gls{gin} layers, each of which includes a two-layer MLP and graph normalization \citep{cai2020graphnorm}. The dimension of the node representations is 32.
The readout function of almost all methods consists of a two-layer MLP and then an add pooling layer.
In \gls{gtp}, the final prediction is obtained by summing the layer-wise predictions, and the readout function is composed of an add pooing layer followed by a linear layer. 
In \gls{gtp}, we employ six hand-crafted transformations. For a fair comparison, \gls{gtl} and \gls{ocgtl} use six learnable graph transformations in all experiments. 
Additional hyperparameter settings are recorded for reproducibility in Appendix B. 

\paragraph{Improved Implementation of \gls{ocgin}.}  
With these implementation details for the \glspl{gnn} we can significantly improve the performance of \gls{ocgin} over the implementation in \citet{zhao2021using}. For this reason, our empirical study includes both \gls{ocgin} (the original version with mean pooling and batch normalization) and \gls{ocgin}$^\dagger$ (our improved version with add pooling and graph normalization).  

\paragraph{Non-GNN-based Baselines.}
Besides \gls{ocpool}, we include four two-stage detection methods proposed by \citet{zhao2021using}.
Two of them use unsupervised graph embedding methods, \gls{g2v} \citep{narayanan2017graph2vec} or \acrshort{fgsd} \citep{verma2017hunt}, to extract graph-level representations. 
The other two of them make use of graph kernels (\gls{wlk} \citep{shervashidze2011weisfeiler} or \gls{pk} \citep{neumann2016propagation}), which measure the similarity between graphs.
For all two-stage detection baselines, we use \gls{ocsvm} (with $\nu=0.1$) as the downstream outlier detector. 

The number of iterations specifies how far neighborhood information can be propagated. By setting the number of iterations to $4$, we get a fair comparison to the \gls{gnn}-based methods, which all have $4$ \gls{gnn} layers. All other hyperparameters correspond to the choices in \citet{zhao2021using}. 

\subsection{Experimental Results}
\paragraph{Summary.} We compare \gls{ocgtl} with all existing baselines on
nine real-world datasets. The detection results in terms of average AUC ($\%$) with standard deviation are reported in \Cref{tab:glad_results,tab:social_network_results}. The results in terms of F1-score are reported in Appendix C. 
We can see that \gls{ocgtl} achieves competitive results on all datasets and has the best average rank of $1.4$. 
On average over nine datasets, \gls{ocgtl} outperforms \gls{ocgin} of \citet{zhao2021using} by $11.8\%$. 
We can conclude that \gls{ocgtl} raises the detection accuracy in graph-level \gls{ad} on various application domains significantly, namely by $9.6\%$ on the bioinformatics domain, by $17.7\%$ on the molecular domain, and by $8\%$ on the social-networks domain.

Moreover, methods with performance flip are marked with a $^*$ in \Cref{tab:glad_results}. 
In \Cref{tab:detailed_results} we report the results of \gls{ocgin} \citep{zhao2021using} and our method \gls{ocgtl} on both experimental variants of datasets where the performance flip is observed. We can see that all existing baselines suffer from the performance flip issue, while \gls{ocgtl} is the only model without performance flip on any of the datasets.  

\paragraph{Ablation Study of Methods.}
Here we discuss the results in \Cref{tab:glad_results} from the perspective of an ablation study to understand if and how the advantages of combing deep \gls{occ} and neural transformation learning complement each other. From results in \Cref{tab:glad_results}, we can see that \gls{ocgtl} improves over \gls{ocgin}$^\dagger$ (with our improved implementation) on 8 of 9 datasets by adding $\mathcal{L}_{\textrm{GTL}}$ as the regularization term and outperforms \gls{gtl} on 8 of 9 datasets by utilizing the Euclidean distance for detection. We can conclude, that the two terms in the loss function of \gls{ocgtl} complement each other and offer two metrics for detecting anomalies. As a result, \gls{ocgtl} consistently outperforms \gls{ocgin} and \gls{gtl}. This is aligned with our theoretical results in \Cref{sec:theory}.

\Gls{gtp} applies hand-crafted graph transformations. Its performance varies across datasets since it is sensitive to the choice of transformations. Even though it works well on the PROTEINS dataset with the chosen graph transformations, its performance on other datasets is not competitive to \gls{ocgtl}. 
Finding the right transformations for each dataset requires domain knowledge, which is not the focus of this work.
In comparison, \gls{ocgtl} learns data-specific transformations automatically and performs consistently well on various datasets.

\paragraph{Study of Design Choices.} 
To raise the bar in deep \gls{ad} on graphs, we have to understand the impact of the design choices associated with the \gls{gnn} architecture. Here we discuss the type of pooling layer for the readout function (\Cref{eqn:readout}) and normalization of the \gls{gnn} layers.

First, we study the impact of different pooling layers. \gls{ocpool} is the ideal testbed to compare add pooling, mean pooling, and max pooling due to its simplicity. Results of running the three options on nine datasets are reported in \Cref{fig:pool}. Add pooling outperforms the other options. This may result from add pooling injecting information about the number of nodes into the graph representations. \gls{ocpool} with add pooling is a simple yet effective method for \gls{ad}. It provides a simple heuristic for aggregating node attributes into graph representations. As shown in \Cref{tab:glad_results}, it does well on many datasets (particularly on the PROTEINS and AIDS datasets), even though it does not account for graph structure (edges).
\begin{figure}[t]
\centering
\includegraphics[width=0.88\linewidth]{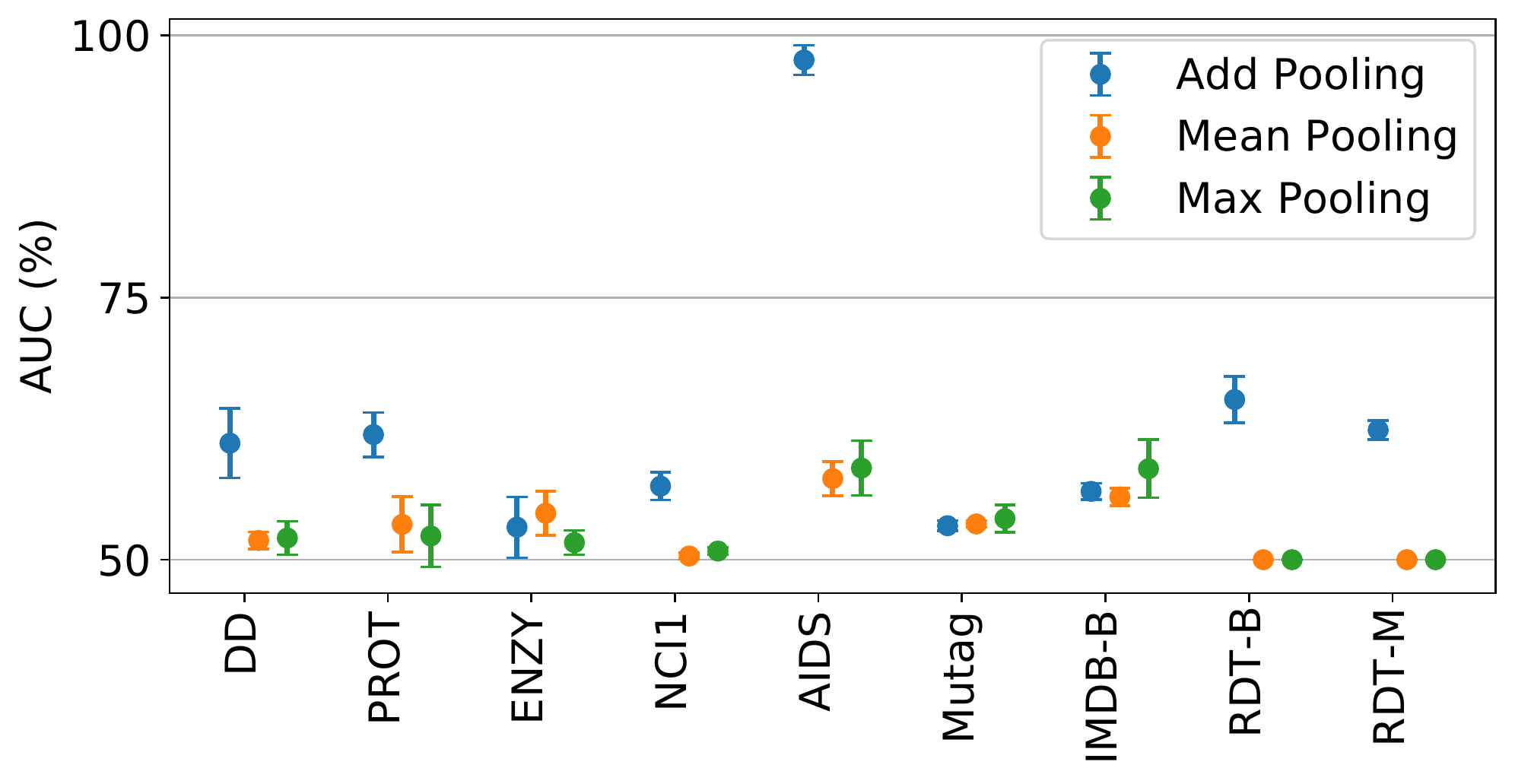}
\caption{\Gls{ocpool} with add pooling (blue) outperforms alternative choices (mean (orange), max (green)). The results on nine datasets are reported in terms of AUC ($\%$). }
\label{fig:pool}
\vspace{-5pt}
\end{figure}

\begin{figure}[t!]
\centering
		\begin{subfigure}[b]{0.75\linewidth}
		\includegraphics[width=\linewidth]{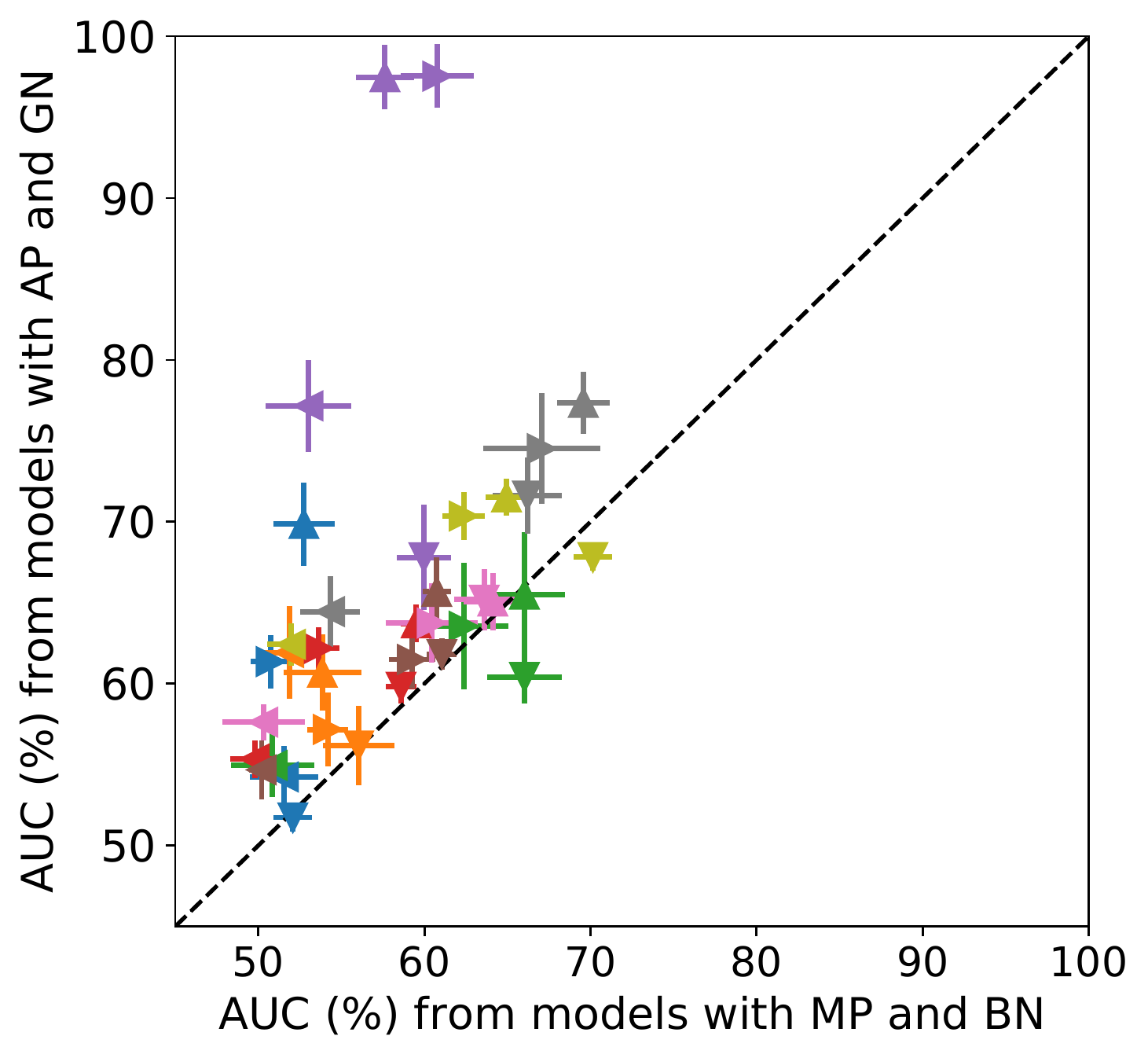}
	\end{subfigure}\\
			\begin{subfigure}[b]{\linewidth}
		\includegraphics[width=\linewidth]{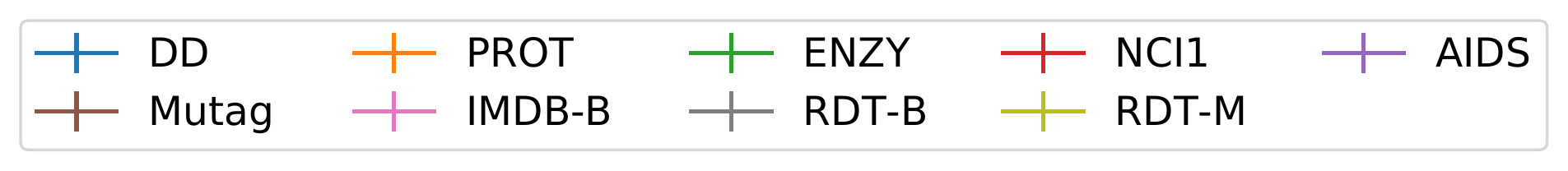}
	\end{subfigure}\\
				\begin{subfigure}[b]{0.8\linewidth}
		\includegraphics[width=\linewidth]{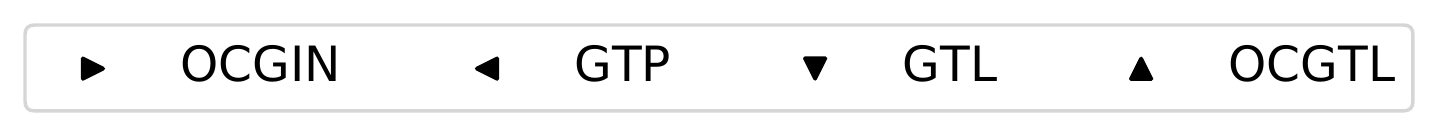}
		\end{subfigure}
    \caption{A comparison of two design choices in \textit{deep} \gls{gnn}-based methods, namely, add pooling with graph normalization (AP + GN, y-axis) and mean pooling with batch normalization (MP + BN, x-axis). Each point compares the detection results of the two variants of one model on one dataset. Most points falling above the diagonal indicate that AP + GN is the preferred design choice for \gls{gnn}-based \gls{ad} methods.}
    \label{fig:design_choice}
    \vspace{-5pt}
  \end{figure}
 
Second, to study the combined impact of the pooling layer and normalization layers on the \textit{deep} \gls{gnn}-based methods,
we compare add pooling with graph normalization (AP + GN) and mean pooling with batch normalization (MP + BN).
In \Cref{fig:design_choice}, we visualize in a scatter plot the performance of all \gls{gnn}-based methods on all nine datasets, contrasting the AP + GN result with the MP + BN result in terms of average AUCs ($\%$) with standard deviations.
Almost all points fall above the diagonal, meaning that AP + GN is preferable to MP + BN, which has been the design choice of \citet{zhao2021using} for \gls{ocgin}. With the new design choices, we are able to significantly raise the bar in graph-level \gls{ad}.

\section{Conclusion}
We develop a novel end-to-end graph-level \gls{ad} method, \gls{ocgtl}, that combines the best of deep \gls{occ} and neural transformation learning. \gls{ocgtl} mitigates the shortcomings of deep \gls{occ} by using graph transformation learning as regularization and complements graph transformation learning by introducing a sensitivity to the norm of the graph representations. 
Our comprehensive empirical study supports our claim and theoretical results. It shows that \gls{ocgtl} performs best in various challenging domains and is the only method not struggling with performance flip.  

\section*{Acknowledgements}
The Bosch Group is carbon neutral. Administration, manufacturing and research activities do no longer leave a carbon footprint. This also includes GPU clusters on which the experiments have been performed.
MK acknowledges the DFG awards KL 2698/2-1 \& KL 2698/5-1 and the BMBF awards 01$|$S18051A, 03$|$B0770E, and 01$|$S21010C.
SM acknowledges support by DARPA under Contract No. HR001120C0021, NSF under grants 2047418, 1928718, 2003237 and 2007719; DOE under grant DE-SC0022331, and gifts from Intel, Disney, and Qualcomm. Any opinions, findings and conclusions or recommendations are those of the authors and do not necessarily reflect the views of DARPA.

\bibliographystyle{named}
\bibliography{refs}

\appendix

\onecolumn
\section*{Appendix}
\section{Datasets Statistics}
\label{app:data_statistic}
We select nine graph classification datasets for the evaluation from three domains (bioinformatics, molecules, and social networks). For each dataset, we report the number of graphs, the dimension of node attributes, the average number of nodes, and the average number of edges in each class as the statistics. We list the statistics of these nine datasets in \Cref{tab:data_stats}.
  \begin{table*}[ht]
  \small
	\caption{The statistics of used datasets. We report the number of graphs, the dimension of node attributes, the average number of nodes, and the average number of edges for each class in each dataset.
	}
	\label{tab:data_stats}
	\centering

	\begin{tabular}{ccccccc}
        \hline
		  Dataset & Category& Class & $\#$Graphs & $\#$NodeAttrs&Avg.$\#$Nodes & Avg.$\#$Edges \\
		\hline
        \multirow{2}{*}{DD} & \multirow{2}{*}{Bioinformatics}& 0&691 & 89& 355.2&1806.6 \\
        &&1 & 487 &89 & 183.7 & 898.9 \\
        \hline
        \multirow{2}{*}{PRTOEINS}& \multirow{2}{*}{Bioinformatics}& 0&663 & 3& 50.0&188.1 \\
        &&1 & 450 &3 & 22.9 & 83.0 \\
        \hline
        \multirow{6}{*}{ENZYMES}& \multirow{6}{*}{Bioinformatics}& 0&100 & 3& 36.2 &132.7 \\
        &&1 & 100 &3 & 29.9 & 113.8 \\
        &&2 & 100 &3 & 28.9 & 111.2 \\
        &&3 & 100 &3 & 38.2 & 148.8 \\
        &&4 & 100 &3 & 31.4 & 119.6\\
        &&5 & 100 &3 & 31.2 & 119.6\\
        \hline
        \multirow{2}{*}{NCI1}& \multirow{2}{*}{Molecules}& 0&2053 & 37& 25.7&55.3 \\
        &&1 & 2057 &37 & 34.1 & 73.9 \\
        \hline
        \multirow{2}{*}{AIDS}& \multirow{2}{*}{Molecules}& 0 &400 &38 &37.6 &80.5\\
        && 1 & 1600 & 38 & 10.2 &20.4\\
        \hline
        \multirow{2}{*}{Mutagenicity}& \multirow{2}{*}{Molecules}& 0 &2401 &14 &29.4 &60.6\\
        && 1 & 1936 & 14 & 31.5 &62.7\\
        \hline
        \multirow{2}{*}{IMDB-B}& \multirow{2}{*}{Social networks}& 0&500 & 136& 20.1 &193.6 \\
        &&1 & 500 &136 & 19.4 & 192.6 \\
        \hline
        \multirow{2}{*}{REDDIT-B}& \multirow{2}{*}{Social networks}& 0 & 1000 & 1 & 641.3 & 1471.9\\
        && 1 & 1000 &1 & 218.0 & 519.1\\
        \hline
                \multirow{5}{*}{REDDIT-M}& \multirow{5}{*}{Social networks}& 0 & 1000 & 1 & 799.5 & 2035.5\\
        && 1 & 1000 &1 & 852.1 & 1940.4\\
        && 2 & 1000 &1 & 374.1 & 856.5\\
        && 3 & 1000 &1 & 249.6 & 534.0\\
        && 4 & 1000 &1 & 267.0 & 581.7\\
        \hline
        \end{tabular}
      \begin{tablenotes}
      \small
      \item  ${}^\star$ In IMDB-B, the one-hot degree is used as node attributes. In REDDIT-B and REDDIT-M, the constant one is used as node attributes.
    \end{tablenotes}
\end{table*}
\section{Additional Implementation Details}
\label{app:exp_details}

\textbf{Training hyperparameters:}
We use the Adam optimizer with an initial learning rate of 0.001 and decay the learning rate by 0.5 every 100 epochs. 
We set the maximum epochs as 500 and the batch size as 128.  
We use the early stopping based on validation loss (without access to the true labeled anomalies) for the training. The early stopping is implemented with a patience parameter of 100 epochs to ease the sensitivity to fluctuations in the validation loss. An early stopping without access to the true anomalies is critical for an unbiased model evaluation in the \gls{ad} tasks.

\textbf{Hardware and Software:}
All models are trained in our GPU cluster, which consists of NVIDIA GeForce GTX TITAN X GPUs, and NVIDIA TITAN X Pascal GPUs. All \gls{gnn}-based models have been implemented by means of the Pytorch Geometrics library. The implementations of \acrshort{g2v} and \acrshort{fgsd} are from Karate Club library, while the implementations of \acrshort{wlk} and \acrshort{pk} are from GraKel library. 
\begin{table*}[t!]
\small
	\caption{Average F1-scores with standard deviations of 9 methods on 9 datasets. \gls{ocgtl} performs best generally.}
	\resizebox{\linewidth}{!}{
	\centering
	\begin{tabular}{c|ccc|ccc|ccc}
 Datasets &DD &  PROT & ENZY & NCI1 & AIDS &  Mutag & IMDB-B  & RDT-B & RDT-M\\

		\hline
WLK& 0.56$\pm$0.016	& 0.58$\pm$0.036	& 0.84$\pm$0.004	& 0.49$\pm$0.008	& 0.48$\pm$0.008	& 0.52$\pm$0.01	& 0.57$\pm$0.032	& 0.5$\pm$0.007	& 0.8$\pm$0.001	\\
PK& 0.54$\pm$0.019	& 0.57$\pm$0.025	& 0.83$\pm$0.002	& 0.51$\pm$0.017	& 0.49$\pm$0.016	& 0.52$\pm$0.014	& 0.52$\pm$0.018	& 0 & 0	\\
G2V& 0.48$\pm$0.026	& 0.5$\pm$0.036	& 0.84$\pm$0.01	& 0.5$\pm$0.007	& 0.76$\pm$0.032	& 0.5$\pm$0.006	& 0.53$\pm$0.024	& 0.5$\pm$0.006	& 0.8$\pm$0.001 \\	
FGSD& 0.63$\pm$0.019	& 0.59$\pm$0.025	& 0.83$\pm$0.01	& 0.54$\pm$0.009	& 0.95$\pm$0.021	& 0.52$\pm$0.009	& 0.56$\pm$0.018	&-	&-			\\
OCGIN& 0.55$\pm$0.022	& 0.55$\pm$0.021	& 0.86$\pm$0.007	& 0.52$\pm$0.013	& 0.49$\pm$0.01	& 0.57$\pm$0.007	& 0.57$\pm$0.025	& 0.63$\pm$0.033	& 0.82$\pm$0.003	\\
\hline
OCGIN$^\dagger$& 0.59$\pm$0.018	& 0.56$\pm$0.029	& 0.86$\pm$0.011	& 0.59$\pm$0.013	& \textbf{0.97$\pm$0.012}	& 0.60$\pm$0.017	& 0.60$\pm$0.02	& 0.68$\pm$0.031	& 0.84$\pm$0.004	\\
OCPool& 0.59$\pm$0.028	& \textbf{0.61$\pm$0.024}	& 0.84$\pm$0.008	& 0.55$\pm$0.009	& \textbf{0.97$\pm$0.017}	& 0.52$\pm$0.009	& 0.53$\pm$0.012	& 0.66$\pm$0.023	& 0.82$\pm$0.002 \\
GTP& 0.52$\pm$0.014	& 0.59$\pm$0.028	& 0.84$\pm$0.005	& 0.54$\pm$0.011	& 0.63$\pm$0.033	& 0.54$\pm$0.018	& 0.55$\pm$0.013	& 0.61$\pm$0.016	& 0.82$\pm$0.003	\\
GTL& 0.57$\pm$0.020	& 0.60$\pm$0.036	& 0.86$\pm$0.007	& 0.57$\pm$0.009	& 0.53$\pm$0.024	& 0.60$\pm$0.007	& 0.60$\pm$0.024	& 0.66$\pm$0.02	& 0.84$\pm$0.003	\\
\hline
OCGTL& \textbf{0.66$\pm$0.025}	& 0.60$\pm$0.025	& \textbf{0.87$\pm$0.012}	& \textbf{0.60$\pm$0.011}	& \textbf{0.97$\pm$0.011}	& \textbf{0.63$\pm$0.02}	& \textbf{0.62$\pm$0.019}	& \textbf{0.7$\pm$0.02}	& \textbf{0.85$\pm$0.004} \\	
        \hline
	\end{tabular}}
	\begin{tablenotes}
      \small
      \item  $^\dagger$ OCGIN is the original implementation from \citet{zhao2021using}, while OCGIN$^\dagger$ denotes our improved implementation.
    \end{tablenotes}
	\label{tab:f1-score}
\end{table*}

\section{Additional Experimental Results}
\label{app:results}
Here we provide additional experimental results, which include two additional comparisons of design choices in \Cref{fig:app_designchoice} and the evaluation results of nine methods on nine datasets in terms of F1-score in \Cref{tab:f1-score}.

\begin{figure}[ht!]
	\centering
	\begin{subfigure}[b]{0.38\linewidth}
	\includegraphics[width=\linewidth]{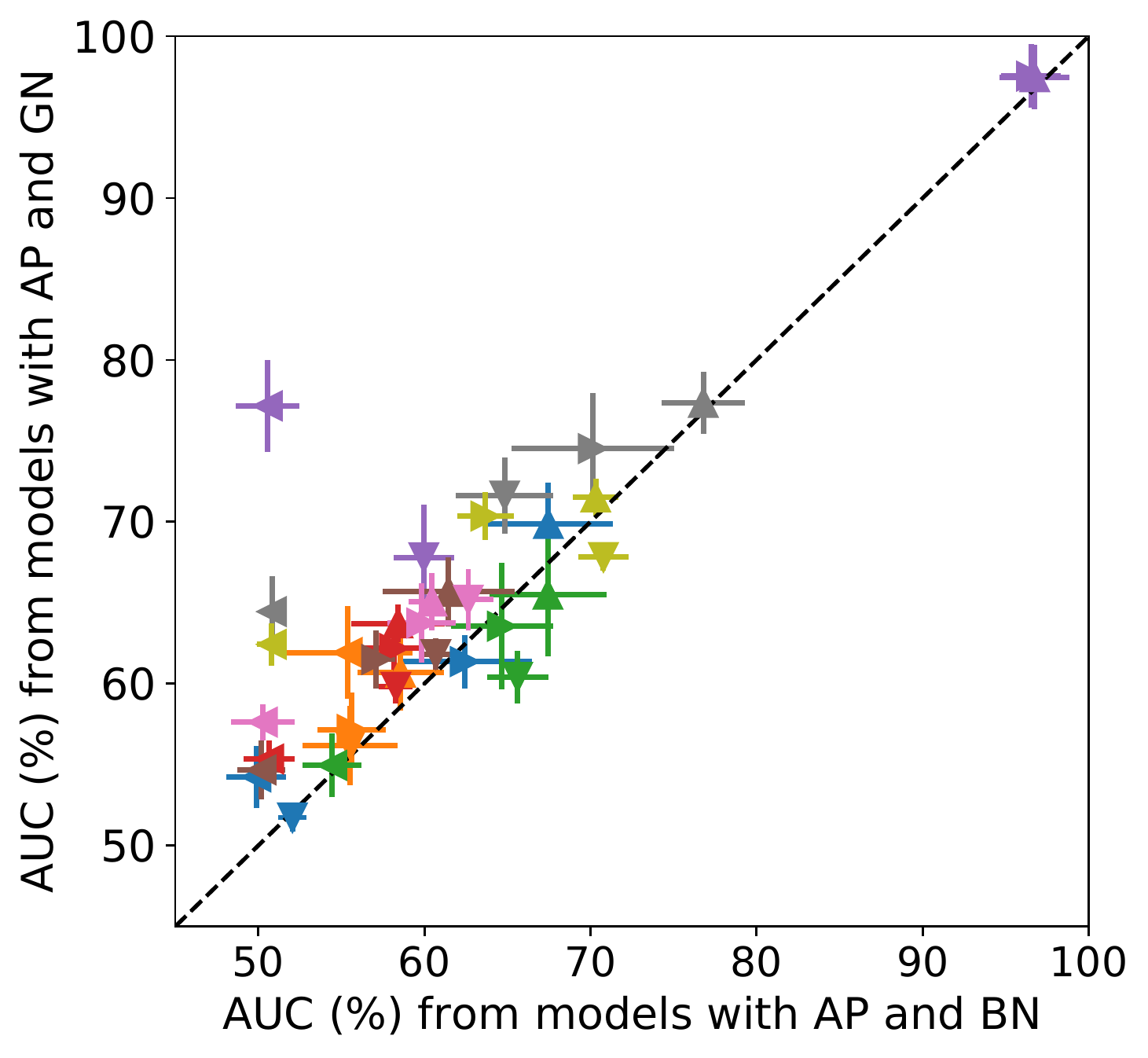}
	\caption{AP+GN v.s. AP+BN}
	\label{fig:gn_vs_bn}
	\end{subfigure}
		\begin{subfigure}[b]{0.38\linewidth}
	\includegraphics[width=\linewidth]{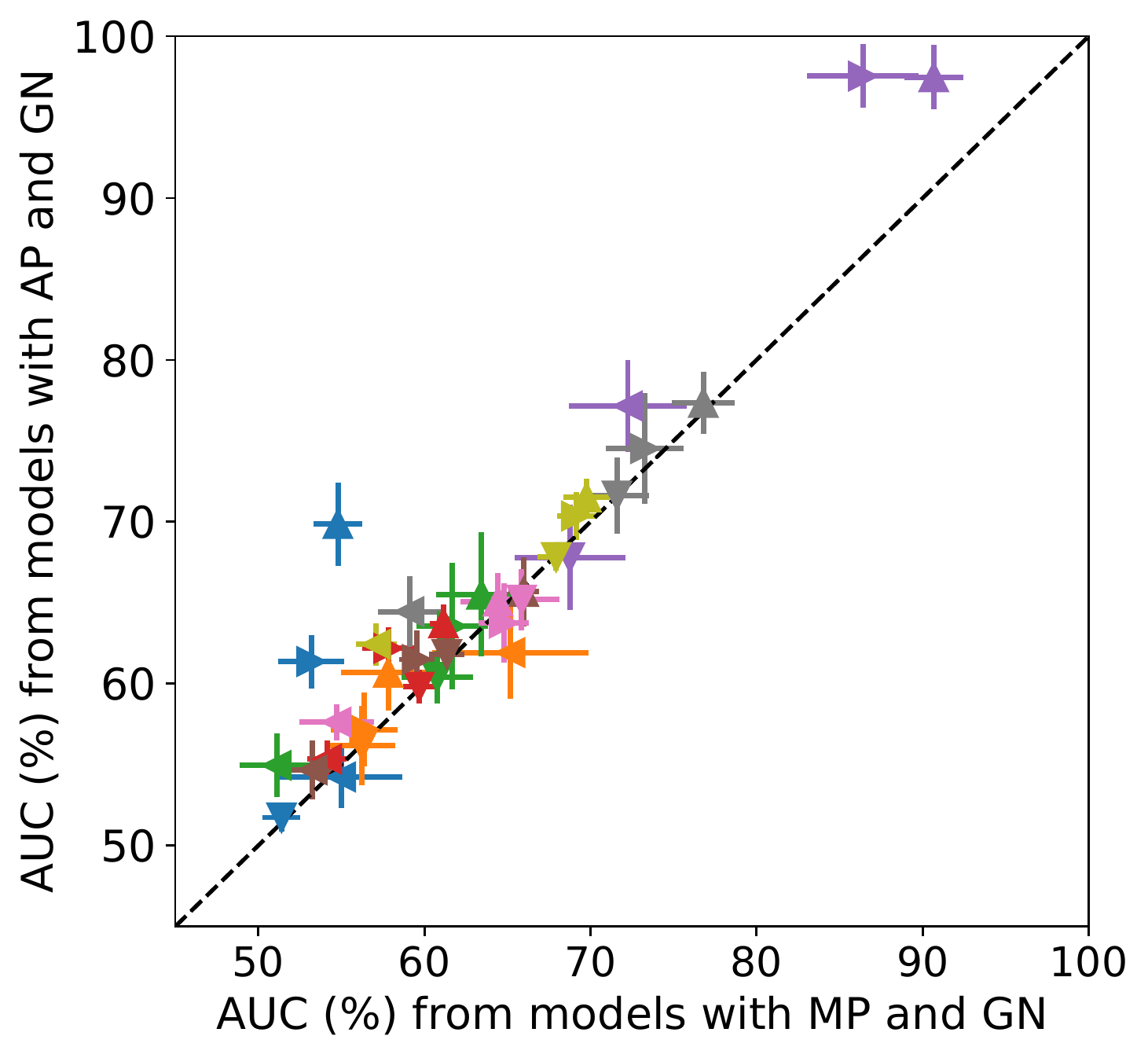}
	\caption{AP+GN v.s. MP+GN}
	\label{fig:ap_vs_mp}
	\end{subfigure}
			\begin{subfigure}[b]{0.7\linewidth}
		\includegraphics[width=\linewidth]{images/legend_data.pdf}
	\end{subfigure}\\
				\begin{subfigure}[b]{0.55\linewidth}
		\includegraphics[width=\linewidth]{images/legend_model.pdf}
		\end{subfigure}
	\caption{A comparison of design choices in \textit{deep} \gls{gnn}-based methods. (a) Results of AP+BN on x-axis against AP+GN on y-axis. (b) Results of MP+GN on x-axis against AP+GN on y-axis. In conclusion, add pooling with graph normalization is preferable.}
	\label{fig:app_designchoice}
\end{figure}

\end{document}